\documentclass{article}

\usepackage[final]{ewrl_2026}

\usepackage[utf8]{inputenc}
\usepackage[T1]{fontenc}
\usepackage{hyperref}
\usepackage{url}
\usepackage{booktabs}
\usepackage{amsfonts}
\usepackage{amsmath}
\usepackage{amsthm}
\usepackage{nicefrac}
\usepackage{microtype}
\usepackage{xcolor}
\usepackage{graphicx}
\usepackage{subfig}
\usepackage{wrapfig}
\usepackage{braket}
\usepackage{tikz}
\usetikzlibrary{positioning, arrows.meta, fit, calc, backgrounds}
\usepackage{algorithm,algpseudocode}
\usepackage{bm}
\usepackage{paralist}

\graphicspath{{../quantum_thesis/}{./}{figures/}}
\definecolor{oiBlue}{HTML}{0072B2}   \definecolor{oiOrange}{HTML}{E69F00}
\definecolor{oiGreen}{HTML}{009E73}  \definecolor{oiVermil}{HTML}{D55E00}
\definecolor{oiSky}{HTML}{56B4E9}    \definecolor{oiPurple}{HTML}{CC79A7}

\tikzset{
  >={Stealth[length=2.8mm,width=3mm]},
  font=\footnotesize,
  enc/.style   ={draw=oiBlue,   fill=oiBlue!8,   rounded corners=4pt, very thick, align=center, inner sep=4pt},
  gru/.style   ={draw=oiOrange, fill=oiOrange!15, rounded corners=4pt, very thick, align=center,
                 minimum width=30mm, minimum height=13mm},
  urnn/.style  ={draw=oiGreen,  fill=oiGreen!15,  rounded corners=4pt, very thick, align=center,
                 minimum width=30mm, minimum height=13mm},
  pol/.style   ={draw=oiVermil, fill=oiVermil!12, rounded corners=4pt, very thick, align=center,
                 font=\small, minimum width=22mm, minimum height=9mm},
  val/.style   ={draw=oiPurple, fill=oiPurple!14, rounded corners=4pt, very thick, align=center,
                 font=\small, minimum width=22mm, minimum height=9mm},
  obsframe/.style={draw=black!55, rounded corners=2pt, inner sep=1.2pt, line width=0.6pt},
  xt/.style    ={circle, draw=none, fill=oiBlue!82!black, text=white,
                 font=\bfseries\large, minimum size=10mm, inner sep=0pt},
  flow/.style  ={line width=1.3pt, draw=black!78, -{Latex[round][length=2.8mm,width=3mm]},
                 shorten >=0.4pt, shorten <=0.4pt},
  wire/.style  ={line width=1.3pt, draw=black!78},             % connector stem, no arrowhead
  loop/.style  ={-{Stealth[length=1.6mm]}, semithick, draw=black!55},
  term/.style  ={font=\normalsize, text=black!75},   % terminal output labels (math)
  lane/.style  ={rounded corners=6pt, inner sep=5mm, draw=none},
  shared/.style={rounded corners=6pt, inner sep=4mm, draw=black!25, line width=0.5pt, fill=black!4},
}
\theoremstyle{definition}
\newtheorem{definition}{Definition}[section]
\newtheorem{theorem}{Theorem}[section]

\newcommand{\figref}[1]{Figure~\ref{#1}}

\newcommand{\eqnref}[1]{Equation~(\ref{#1})}
\usepackage{paralist}

\title{Reinforcement Learning with Complex (valued) Memories}

\author{
  Sathya Kamesh Bhethanabhotla\thanks{Work done partially at the University of Freiburg.} \\
  University of Amsterdam \\
  \texttt{s.k.bhethanabhotla@uva.nl}
  \And
  Efstratios Gavves \\
  University of Amsterdam
  \And
  André Biedenkapp\footnotemark[1] \\
  Karlsruhe Institute of Technology \\
  \texttt{biedenkapp@kit.edu}
}

\begin{document}

\maketitle

% The author \thanks footnote steps the global footnote counter, and this style does not
% reset it after the title; reset here so body footnotes stay numbered from 1.
\setcounter{footnote}{0}

% ============================================================================
% ABSTRACT
% ============================================================================

% more hype, call out tasks where it imrpoves on some kind of memory tasks, and PO needs to be better motivated, and close to IRL agents. talk about how
\begin{abstract}
Partially observable environments pose a fundamental challenge in deep reinforcement learning, requiring agents to compress temporal information from observations and maintain a memory to make effective decisions. While there exist many approaches ranging from gated recurrence to attention mechanisms and model-based RL, the search for effective representational techniques that can capture long-term dependencies remains an active area of research. In this work we revisit Unitary recurrent networks (uRNNs) \citep{arjovsky2016urnn, pmlr-v70-jing17a}, that demonstrated superior gradient flow and associative recall, expressing the recurrence and the hidden state in a complex vector space. Their norm preserving unitary dynamics enable information propagation through long sequences. To this end, we propose three different versions of uRNNs as drop-in replacements for recurrent PPO architectures, and demonstrate that the simple recurrence and the added degree of freedom from the phase of the complex representations enable significant gains over baselines on several memory-improvable tasks, including continuous control. We further explore how to preserve the phase information of the complex hidden state for a \textit{phase-aware} policy by drawing a parallel to how quantum states are measured. With our methods reaching up to $2$-$3\times$ the reward in environments like rocksample and Craftax compared to the baselines, this work points towards an exciting new direction of representations for RL and the problem of partial observability. Code is available at: \url{https://github.com/Sathya98/qurl}

% feels weird - cannot connect the revisiting URNNs sentence wuth the next one

\end{abstract}

% ============================================================================
% 1. INTRODUCTION
% ============================================================================
\section{Introduction}\label{sec:introduction}

\if false
Suggested Intro structure where the "top-level" items are 1-2 paragraphs
\begin{itemize}
    \item RL is great but struggles in POMDPs
    \begin{itemize}
        \item A) Quick mention that RL is great for sequential decision making and give examples of successes
        \item B) Mention that decisions often require information of past events which leads to a POMDP setting and memory requirements for agents
    \end{itemize}
    \item Discuss how RL commonly deals with this issue
    \begin{itemize}
        \item e.g. Frame stacking and histories in general
        \item architectures such as RNNs
    \end{itemize}
    \item Point out why we believe that this is not enough.
    \begin{itemize}
        \item dealing with large histories makes learning difficult/intractable
        \item RNNs or similar networks have architectural choices which likely are not setup for the RL problem (e.g,., gating might need to be highly specialized for various situations)
    \end{itemize}
    \item What we do differently and why
    \begin{itemize}
        \item We evaluate a "memory based" architecture that, to the best of our knowledge, has not yet been explored in RL
        \item We extend this approach by observing that the complex-valued structure of the network allows us "approximate a wave function" rather than just distributions.
    \end{itemize}
    \item List of contributions
\end{itemize}
\fi

The human mind fundamentally relies on memory to function effectively in our partially observable world. When navigating through an environment, we do not have access to complete information about our surroundings at any given moment. Instead, we must integrate observations over time, maintaining an internal representation of relevant information from the past to inform our decisions and actions. Memory is what makes action possible. Without it, we could neither track a ball in motion, recall where we left our keys hours earlier nor any other task that unfolds over time. In our effort to understand and replicate the mechanics of sequential decision-making, we have turned to frameworks like reinforcement learning \citep[RL; ][]{sutton1998}, developing methods that can emulate these aspects of agency. 
%% citation with the neuro science people for memory.

Deep RL has addressed partial observability \citep{astrom1965optimal,Kaelbling1998PlanningAA} and memory in a variety of ways in the past decades of work \citep[see, e.g.][]{lambrechts-tmlr22a,ni-icml22a}. We have also seen the development of various environments that can be used to test different aspects of these problems. From early work on value-based deep learning methods like DQN \citep{mnih2015humanlevel} and the recurrent DRQN \citep{hausknecht2017deeprecurrent} playing Atari games \citep{bellemare13arcade}, over policy gradient methods like Trust Region Policy Optimization (TRPO) \citep{schulman2017trustregionpolicyoptimization} and Proximal Policy Optimization (PPO) \citep{schulman2017proximalpolicyoptimizationalgorithms}, to transformer-based methods \citep{chen2021decision} we have seen a lot of progress in the field. Further, model-based approaches such as Dreamer~\citep{hafner2025dreamerv3}, PETS~\citep{chua-neurips18a} or Recall2Imagine~\citep{samsami2024mastering} build an internal models of the world and use those to plan, enabling them to tackle partial observability. With techniques like gated recurrence \citep{hochreiter1997lstm}, attention \citep{vaswani2017attention}, and state space models \citep{gu2024mambalineartimesequencemodeling}, the search for effective representational techniques to compress state information over time remains an active area of research.
Despite this progress, the field still struggles with learning instabilities, for example caused by vanishing gradients or the dependence on short horizons.

A decade ago, Unitary Evolution Recurrent Neural Networks \citep[uRNNs;][]{arjovsky2016urnn} offered a fresh approach to addressing the challenges of memory and gradient flow in recurrent architectures. uRNNs operate in a complex-valued vector space, using complex-valued parameters and hidden states. The central insight of uRNNs is that constraining the recurrent dynamics to be \emph{unitary}, a norm-preserving transformation that ensures stable evolution of the hidden state over arbitrarily long sequences, and aids associative recall. This extension to the complex domain provides an additional degree of freedom in the form of a phase, allowing for richer representational capacity. In this work, we propose this class of recurrence as a representation paradigm for RL agents acting in partially observable tasks, requiring long term memory.
We show how different variants of this recurrent class can be straightforwardly integrated into the recurrent PPO method.
Instead of learning real-valued representations these drop-in replacements enable learning complex-valued representations.
We empirically show that such complex-valued representations already achieve a significantly better performance across a range of memory improvable tasks. %add a sentence about how craftax, a really hard game results in a massive gain in performance

While complex-valued representations already enable PPO to learn better policies, we further aim to leverage the phase information directly. We propose a novel approach  to propagate the phases captured in the complex hidden state to a discrete stochastic policy. By allowing the phases of the complex parameters to interfere constructively and destructively through a Born-rule like sampling \citep{born1984quantenmechanik}, akin to collapsing a wavefunction, we highlight a potentially deeper connection between policies and quantum states. Evaluating this \textit{Phase Aware} policy shows competitive performance to other baselines over tasks like T-maze, battleship and craftax. Particularly in long-horizon tasks, our approach outperforms canonical recurrent architectures. Our contributions are: %of this paper are summarized as follows:

%The mathematical foundations underlying unitary transformations extend beyond machine learning into the realm of quantum information theory and quantum mechanics. In quantum mechanics, states are expressed in complex Hilbert spaces, and their evolution through time, governed by solutions to the Schr\"{o}dinger equation, is unitary, preserving both norm and information. The parallel between domains suggests a deeper connection that may offer new insights into the applicability of these representations in reinforcement learning.

%In this work, we explore the unification of these concepts by evaluating the efficacy of uRNNs to address memory and partial observability in reinforcement learning. We hypothesize that the complex vector space, combined with unitary transformations can provide a rich representation space for learning effective policies in partially observable environments.

\begin{enumerate}[I)]
    \item We propose a recurrent PPO approach that uses the uRNN  as the recurrent cell;
    \item We extend the uRNN architecture to enable an input-dependent unitary transformation matrix;
    \item We propose a way to interpret the policy from a quantum information perspective, and introduce a completely complex-valued policy head with a sampling method inspired by the Born rule in quantum mechanics;
    \item We evaluate the proposed methods on a range of partially observable tasks, and show that the uRNN architecture is capable of solving these tasks and outperforms the baselines.
\end{enumerate}
%The results of our analysis show that complex-valued policies and hidden-state representations are a highly promising alternative to real-valued approaches.

% \textbf{Paper outline.} Section~\ref{sec:background} provides the necessary background on reinforcement learning, partial observability, complex algebra, unitary matrices, and uRNNs. Section~\ref{sec:related} reviews related work on memory benchmarks, model-based reinforcement learning, and complex-valued neural networks. Section~\ref{sec:methodology} describes our methodology for integrating uRNNs into PPO, the input-conditional formulation, and the quantum-inspired policy head. Section~\ref{sec:experiments} presents our experimental evaluation and results. Finally, Section~\ref{sec:disc_conc} discusses the implications of our findings, limitations, and directions for future research.

% ============================================================================
% 2. BACKGROUND
% ============================================================================
\section{Background}\label{sec:background}

% This section provides the reader with the necessary background to understand the main approach. We briefly cover Reinforcement Learning and PPO, followed by partial observability and memory, the mathematical foundations of complex algebra and unitary matrices, and finally Unitary Evolution Recurrent Neural Networks.

\subsection{Partial Observability, Memory and the problem setting}\label{sec:pomem}

In many real-world applications, the state is only partially observable, formalized as a Partially Observable Markov Decision Process (POMDP) \citep{astrom1965optimal}. A POMDP augments the standard Markov Decision Process with an observation model and is described by the tuple $(\mathcal{S}, \mathcal{A}, T, R, \Omega, O, \gamma)$, where $\mathcal{S}$ is the set of states and $\mathcal{A}$ the set of actions;
$T(s_{t+1} \mid s_t, a_t)$ is the state transition distribution;
$R(s_t, a_t)$ is the reward function;
$\Omega$ is the set of observations and $O(o_t \mid s_t)$ the observation model;
$\gamma \in [0,1)$ is the discount factor.

Unlike the fully observable case, the agent never accesses the underlying state $s_t$ directly; at each step it receives only an observation $o_t \in \Omega$ emitted from $s_t$ through $O$. A single observation is in general non-Markovian, in that it does not on its own summarize the information needed to act optimally, so a policy $\pi(a_t \mid o_t)$ conditioned on the latest observation alone is insufficient. To recover the optimal policy, the agent must instead estimate the latent state from the history of observations seen so far, $o_{1:t} = (o_1, o_2, \dots, o_t)$, compressing this history into a representation sufficient for action selection. This introduces the need for agents to maintain memory of past observations, and it is the partially observable setting we operate in throughout this work.

The simplest approach is frame stacking \citep{bellemare13arcade}, also employed by \citet{mnih2015humanlevel} in DQN for Atari, where the last four frames were stacked as input. However, this captures dependencies only within the fixed window. Recurrent networks address this by compressing temporal information into a learned hidden state. DRQN \citep{hausknecht2017deeprecurrent} replaced frame stacking with LSTM layers that process single frames while maintaining a hidden state, matching DQN's performance while being more robust to partial observability.

Recurrent architectures still face limitations: sequential processing prevents parallelization and fixed-size hidden states bottleneck information flow. GTrXL \citep{parisotto2020stabilizing} stabilized Transformers for RL through gating mechanisms. Decision Transformer \citep{chen2021decision} reframed RL as sequence modeling by conditioning on desired returns.

\subsection{Unitary Evolution Recurrent Neural Networks}\label{sec:urnns}

Recurrent neural networks suffer from well-known stability issues: when the recurrent weight matrix has eigenvalues with magnitude not equal to one, the hidden state dynamics either contract or blow up exponentially, leading to vanishing and exploding gradients \citep{bengio1994learning}. Gated architectures such as LSTMs \citep{hochreiter1997lstm} mitigate this through learned gating mechanisms, but still lack an inductive bias to bound the gradients.

The introduction of \emph{Unitary Evolution Recurrent Neural Networks} (uRNNs) by \citet{arjovsky2016urnn} directly targeted this gap. The central insight is that if the recurrent transition is constrained to be \emph{unitary}, then the hidden state evolves through a norm-preserving transformation at every time step. A unitary matrix $U \in \mathbb{C}^{n \times n}$, satisfying $U U^\dagger = U^\dagger U = I$, is the complex analog of an orthogonal matrix $Q \in \mathbb{R}^{n \times n}$ (for which $Q Q^\top = I$). Refer to \ref{sec:compalg} for an involved introduction.

\subparagraph{Unitary Recurrent Dynamics.}
A uRNN maintains a hidden state $h_t \in \mathbb{C}^n$ that evolves as:
\begin{equation}
h_{t+1} = \sigma(U h_t + V_x x_t),
\label{eq:urnn_recurrence}
\end{equation}
where $U \in \mathbb{C}^{n \times n}$ is a unitary matrix, $V_x$ maps real-valued inputs into the complex hidden space, and $\sigma$ is the modReLU activation:
\begin{equation}
\sigma_{\text{modReLU}}(z) = \begin{cases}
(|z| + b) \frac{z}{|z|} & \text{if } |z| + b \geq 0, \\
0 & \text{if } |z| + b < 0.
\end{cases}
\label{eq:modrelu}
\end{equation}

\subparagraph{Efficient Parameterization.}
The unitary matrix is decomposed into structured factors:
\begin{equation}
U = D_3 \, R_2 \, \mathcal{F}^{-1} \, D_2 \, \Pi \, R_1 \, \mathcal{F} \, D_1,
\label{eq:urnn_parameterization}
\end{equation}
where $D_1, D_2, D_3$ are diagonal phase matrices, $\mathcal{F}$ is the DFT, $R_1, R_2$ are Householder reflections, and $\Pi$ is a fixed random permutation. Each component is unitary by construction, giving $O(n \log n)$ complexity. Gradient updates are performed on unrestricted parameters, guaranteeing unitarity without projection.
uRNNs demonstrated remarkable performance on associative recall tasks, outperforming LSTMs with a fraction of the parameters. This makes the architecture appealing for RL, where memory-related tasks require integrating information over long sequences.

% ============================================================================
% 3. RELATED WORK
% ============================================================================
\section{Related Work}\label{sec:related}

\paragraph{Benchmarking Memory and Reinforcement Learning}

Works such as \citet{hausknecht2017deeprecurrent, parisotto2020stabilizing, chen2021decision} presented methods relying on recurrence and transformers to handle partial observability. Notably, \citet{ni2023when} showed that transformers can be effective in capturing memory but are unable to assign credit correctly to past actions. Memory Gym \citep{pleines2023memory} and POPGym \citep{morad2023popgym} provide benchmark suites of partially observable environments requiring memory. RLBenchNet \citep{smirnov2025rlbenchnetrightnetworkright} presents an evaluation suite of PPO agents across various architectures.

In our work, we use the POBAX benchmark \citep{tao2025pobax} for our experiments, baselines and comparison. POBAX a benchmark for a suite of partially observable and memory improvable tasks written in JAX \citep{jax2018github}, including tasks ranging from vector input continuous control to complicated image based tasks like Crafter \citep{hafner2021crafter}. They compare the performance of a GRU based recurrent PPO and a Lambda Discrepancy PPO \citep{allenkirtlandtao2024lambdadiscrep} method against a memoryless and a full state baseline.

\paragraph{Model-Based Reinforcement Learning}
Model-Based RL relies on learning a world model \citep{ha2018worldmodels} that can represent the state space of partially observable environments. These models can learn the dynamics of the environment in terms of this compressed latent state $p(s_{t+1}| s_t, a_t)$, enabling planning and off policy RL methods like SAC \citep{haarnoja2017soft}, and TD-MPC2 \citep{hansen2024tdmpc2}. Most notable are PlaNet and Dreamer \citep{hafner2019planet, hafner2025dreamerv3}, employing RSSM-based GRU cells. Recall2Imagine \citep{samsami2024mastering} proposed a world model based on S4 \citep{gu2021combining} for long-term memory. Interestingly, S4 and Mamba \citep{gu2024mambalineartimesequencemodeling} are based on the same kind of recurrence as the uRNN, differing only in the applied non-linearity.

\paragraph{Complex-Valued Neural Networks}
Neural networks with complex parameters have been explored for many years. \citet{amin2011wirtinger} showed how Wirtinger calculus enables gradient descent for complex-valued networks. Follow-up works \citep{trabelsi2018deepcomplex, sarroff2015learningrep, geuchen2023optimal} demonstrated capabilities across domains. For a comprehensive overview, see \citet{bassey2021surveycomplexvaluedneuralnetworks}.

\paragraph{Successors of uRNNs}
Several works addressed the limitation that the original uRNN's parameterization does not span the entire unitary group. \citet{wisdom2016fullcapacity} proposed full-capacity unitary RNNs on Stiefel manifolds. \citet{mhammedi2017efficient} parameterized real orthogonal matrices via Householder reflections (oRNN). \citet{helfrich2018orthogonal} introduced scoRNN via the scaled Cayley transform. \citet{jing2017goru} extended unitary RNNs with gating in GORU, addressing the lack of an explicit forget mechanism. Later, the LRU \citep{lruorvieto23a} was introduced as a linear RNN with complex hidden states, which was the extended in \citet{elelimy2024realtime} for real-time RL. We use the original uRNN architecture as a starting point for exploring unitary transformations in RL. 

In this work, we particularly concentrate on the method presented in \citep{pmlr-v70-jing17a}, a tuneable version of the uRNN composed of unitary transforms on alternate dimensions, that can span the whole unitary space for a given dimensionality. We explore this construction as an additional method for complex-valued recurrence, owing to its simple design and tuneable expressivity.

% ============================================================================
% 4. METHODOLOGY
% ============================================================================
\section{Unitary Recurrent Networks for PPO}\label{sec:methodology}

%\subsection{Motivation and Overview}

We equip a PPO agent with a unitary recurrent cell, hypothesizing that a complex hidden state evolved by norm-preserving transformations provides a stable and expressive memory for partially observable control. Complex data types, forward and backward passes are fully supported in JAX and PyTorch, which we rely on, for building the method (see Appendix \ref{sec:cvnns}).

\subsection{Integrating the uRNN into Proximal Policy Optimization}\label{sec:ppo_urnn}

We build on the recurrent PPO implementation of POBax \citep{tao2025pobax}, where the base recurrent cell is a GRU \citep{cho-gru}. A task-specific encoder maps each observation $o_t$ to a feature vector $x_t$, which drives a recurrent cell whose hidden state $h_t$ feeds separate policy and value heads.

\begin{figure}[t]
    \centering
    \resizebox{0.85\textwidth}{!}{\input{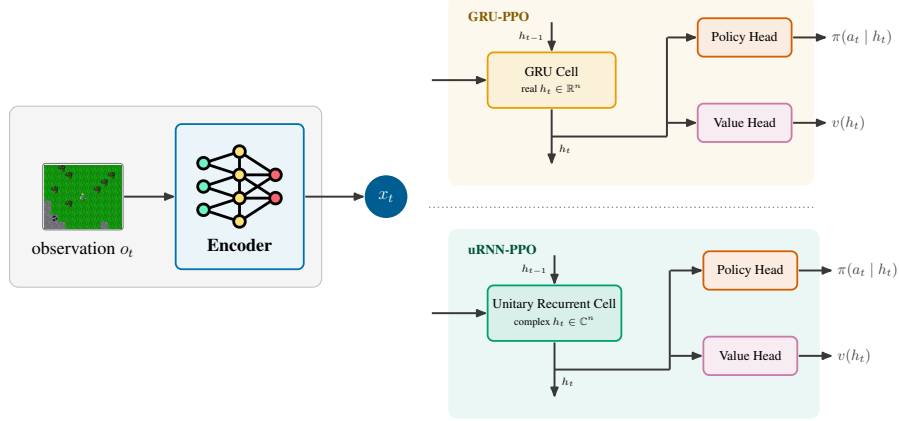}}
    \caption{\textbf{Unitary Recurrent PPO architectures.} The Unitary Recurrent cell can be simply dropped-in for any standard recurrent cell. We show how the different variants can be constructed with complex hidden states, and later, fully complex policies.}
    \label{fig:ppo_urnn_comparison}
\end{figure}

We replace this GRU with a unitary recurrent cell carrying a complex hidden state $h_t \in \mathbb{C}^{n}$, where $n$ matches the GRU baseline's per-environment hidden size (Table~\ref{tab:env_settings}), leaving the rest of the agent unchanged (\figref{fig:ppo_urnn_comparison}). We consider two parameterizations of the unitary transition $U$. In its original form \citep{arjovsky2016urnn}, the recurrence follows \eqnref{eq:urnn_recurrence} with $U$ structured as in \eqnref{eq:urnn_parameterization}: $U$ is constant across time and the input enters only through $V_x$. Native \texttt{complex64} support allows a much simpler implementation than the original; the components are:
\begin{compactitem}
    \item {$D_1, D_2, D_3$:} Diagonal matrices parameterized with learnable phases: $D_{j,j} = e^{i w_{j}}, \; w_j \in \mathbb{R}$.
    \item {$R_1, R_2$:} Reflection matrices with learnable $v_1, v_2 \in \mathbb{C}^{n}$: $R_i = I - 2 v_k v_k^\ast/\|v_k\|^2$.
    \item {$V_x$:} A \texttt{complex64} linear layer mapping the feature vector to $\mathbb{C}^n$ (no bias).
    \item {modReLU:} Implemented as per \eqnref{eq:modrelu}.
\end{compactitem}

The policy and value networks receive $2n$ input features (real and imaginary parts of the hidden state). This construction, due to its time/step independent recurrent dynamics, takes up far fewer parameters than a GRU cell, even though the complex parameters use twice the floating point space.

\paragraph{Parameter Initialization.} We follow the original initialization \citep{arjovsky2016urnn}: $D$ phases from $\mathcal{U}(-\pi, \pi)$, reflection vectors from $\mathcal{U}(-1, 1)$ for both components, $V_x$ via Glorot initialization \citep{glorot10a}, and modReLU bias initialized to 0.

\paragraph{Hidden State Initialization.} We propose initializing the hidden state with a constant complex vector such that $\|h_0\|_2 = 1$:
\begin{equation}
    h_0[k] = \frac{1}{\sqrt{2n}} \left(1 + i\right) \quad k \in \{1, 2, \ldots, n\}
\end{equation}
This is equivalent to an equal superposition state in an $n$-dimensional eigenbasis (Appendix ~\ref{app:qm}), pointing to an interesting interpretation explored in Section~\ref{sec:qm_policy_update}.

\subsection{Tunable Unitary Recurrence with EUNNs}\label{sec:eunn}
% rewrite this quickly for a more intuitive understanding
The original parameterization (\eqnref{eq:urnn_parameterization}) is efficient but spans only a subset of the unitary group \citep{wisdom2016fullcapacity}. As a second method we adopt the Efficient Unitary Neural Network (EUNN) of \citet{pmlr-v70-jing17a}, replacing the fixed FFT structure with a product of $L$ layers whose expressivity is directly tunable:
\begin{equation}
    U_{\text{EUNN}} = D \, F_L F_{L-1} \cdots F_1,
    \label{eq:eunn}
\end{equation}
where $D$ is a diagonal phase and each $F_k$ is block-diagonal in $H/2$ independent $2\times 2$ unitary rotations, each acting on a pair of hidden dimensions $(a,b)$:
\begin{equation}
    \begin{pmatrix} h'_a \\ h'_b \end{pmatrix} =
    \begin{pmatrix} e^{i\phi}\cos\theta & -e^{i\phi}\sin\theta \\[2pt] \sin\theta & \cos\theta \end{pmatrix}
    \begin{pmatrix} h_a \\ h_b \end{pmatrix},
    \label{eq:eunn_block}
\end{equation}
with per-pair angles $\theta, \phi$. Consecutive layers alternate which dimensions are paired via a cyclic shift, mixing information across the full state. The capacity $L$ tunes expressivity: small $L$ restricts $U$ to a structured subspace, while $L=H$ provably spans the full unitary group $U(H)$. The recurrence is again \eqnref{eq:urnn_recurrence} with $U=U_{\text{EUNN}}$, using learnable, input-independent angles.

\subsection{Conditioning the Unitary Transformation on the Observation}\label{sec:conditioning_urnn}

While a constant $U$ ensures stable gradient flow and is parameter-efficient, we hypothesize that an observation-dependent $U(x_t)$ can improve representational capacity. This mirrors the selectivity of modern state-space models, where the transition is made input-dependent \citep{katharopoulos2020lin, gu2024mambalineartimesequencemodeling, yang2024gdn, yang2024deltanet, siems2025deltaproduct}.

We propose the ICuRNN (Input Conditional Unitary RNN) that conditions $U$ on the input $x_t$:
\begin{compactitem}
    \item {$D_1, D_2, D_3$:} Now parameterized with a real-valued linear layer: $D_j = \text{Linear}(x_t)$.
    \item {$R_1, R_2$:} Parameterized with a \texttt{complex64} linear layer: $v_k = \text{ComplexLinear}(x_t)$.
\end{compactitem}
The remaining components ($V_x$, hidden state initialization) are unchanged. While this adds more parameters, we hypothesize that it allows the policy to learn what to remember faster.

% \paragraph{Summary of changes.} Relative to a standard recurrent PPO agent, we (i) swap the GRU for a complex-valued unitary cell, in either the original \citep{arjovsky2016urnn} or tunable EUNN \citep{pmlr-v70-jing17a} parameterization; (ii) initialize the hidden state as an equal-superposition state; (iii) optionally condition the transition on the observation (ICuRNN); and (iv) train the complex and real parameter groups with separate learning rates ($8\times10^{-5}$ and $2.5\times10^{-4}$). The encoder, heads, and PPO objective are otherwise unchanged.

% ============================================================================
% 5. EXPERIMENTS
% ============================================================================
\section{Do uRNNs improve memory capabilities?}\label{sec:experiments}

\subsection{Environments}\label{sec:environments_tasks}

We evaluate on seven tasks from the POBax suite \citep{tao2025pobax}, each partially observable but solvable with the right memory, spanning discrete recall, noisy state tracking, spatial mapping, and continuous control:

\textbf{T-Maze.} The agent receives a cue at the start of a corridor indicating which direction to turn at the far end, but the cue is absent from the observation thereafter. It must carry a single bit across the entire corridor, isolating long-horizon recall. We use a corridor of length 75 (\texttt{tmaze\_75}) to make it a true test of long term memory capability, compared to the original benchmark's 10.

\textbf{RockSample.} A rover on an $n \times n$ grid must sample good rocks and avoid bad ones, using a noisy sensor whose accuracy decays with distance. Solving it requires integrating repeated noisy checks to track the latent quality of each rock. We evaluate the $(11,11)$ and $(15,15)$ grid size configurations.

\textbf{Battleship.} The agent fires on a $10 \times 10$ board with hidden ships and observes only a hit or miss each step, without past outcomes retained in the observation. It must build and maintain an internal map of where it has already fired and what has been hit.

\textbf{Masked Walker and HalfCheetah.} Continuous-control locomotion tasks (Brax) in which only velocity features are observed and all positional state is masked. The agent must integrate the velocity history to recover position, testing memory in a high-dimensional continuous control setting.

\textbf{Craftax (No Inventory).} A pixel-based open-world survival task originally proposed in \citet{hafner2021crafter} and made harder by \citet{matthews2024craftax}. The version we run is a futher modified version by \citet{tao2025pobax}, in which the inventory panel is cropped from the observation. The agent must remember its collected resources and progress through the achievement tree from visual input alone, requiring higher memory capabilities than the original Craftax task. We downsample the observation image size, just as described in the POBAX benchmark. %quote appendix

\subsection{Experimental Setup}\label{sec:exp_setup}

Our implementations of the unitary recurrent networks are built on top of the POBAX benchmark \citep{tao2025pobax}, written entirely in JAX.
We compare against two baselines from POBAX: the GRU recurrent PPO agent, and PPO-LD, which augments recurrent PPO with the $\lambda$-discrepancy memory objective based on a second value function approximator \citep{allenkirtlandtao2024lambdadiscrep}. All architectural details apart from the RNN, including last action concatenation are kept here. Each uRNN variant is a drop-in replacement of the GRU cell, leaving the encoder, heads, and PPO objective untouched. We evaluate the three unitary parameterizations introduced in Section~\ref{sec:methodology}:
\begin{compactenum}
    \item \textbf{uRNN}: original parameterization with a constant transition $U$ (Section~\ref{sec:ppo_urnn}).
    \item \textbf{EUNN}: tunable parameterization with capacity $L = 2$ (Section~\ref{sec:eunn}).
    \item \textbf{ICuRNN}: input-conditioned transition $U(x_t)$ (Section~\ref{sec:conditioning_urnn}).
\end{compactenum}

All uRNN variants use a complex hidden state matching the GRU baseline's hidden size, and use the same encoder, policy and value heads. For the EUNN variant, we only evaluate the simplest tuning setting, with $L=2$, as an initial demonstration given the original work \citep{jing2017tunablenetworks} shows that this setting is already comparable with traditional recurrent cells. We leave the investigation of the effects of tuning to future work.
All evaluations are done for the same total steps as in the POBAX benchmark, averaged over 5 seeds, except for craftax, where we limit the runs to 100M steps.

\paragraph{Hyperparameters} We evaluate the baselines GRU and PPO-LD on the same tuned hyperparameter settings in POBAX. We observe that the complex valued parameters in our methods require a much smaller and separate learning rate than the real valued ones. We keep the network sizes, hidden state sizes, discount factor and number of parallel environments identical to the baselines, but find that a larger number of steps (256) per PPO update works well for the uRNNs. The optimal GAE lambda, learning rates and entropy coefficient were found through sweeps over a grid, and the details per environment can be found in App. \ref{app:hyperparameters}. We do not tune the settings for the No Inventory Craftax environment, given its large number of parallel environments and compute constraints.

\begin{figure}[t]
    \centering
    \includegraphics[width=0.95\textwidth]{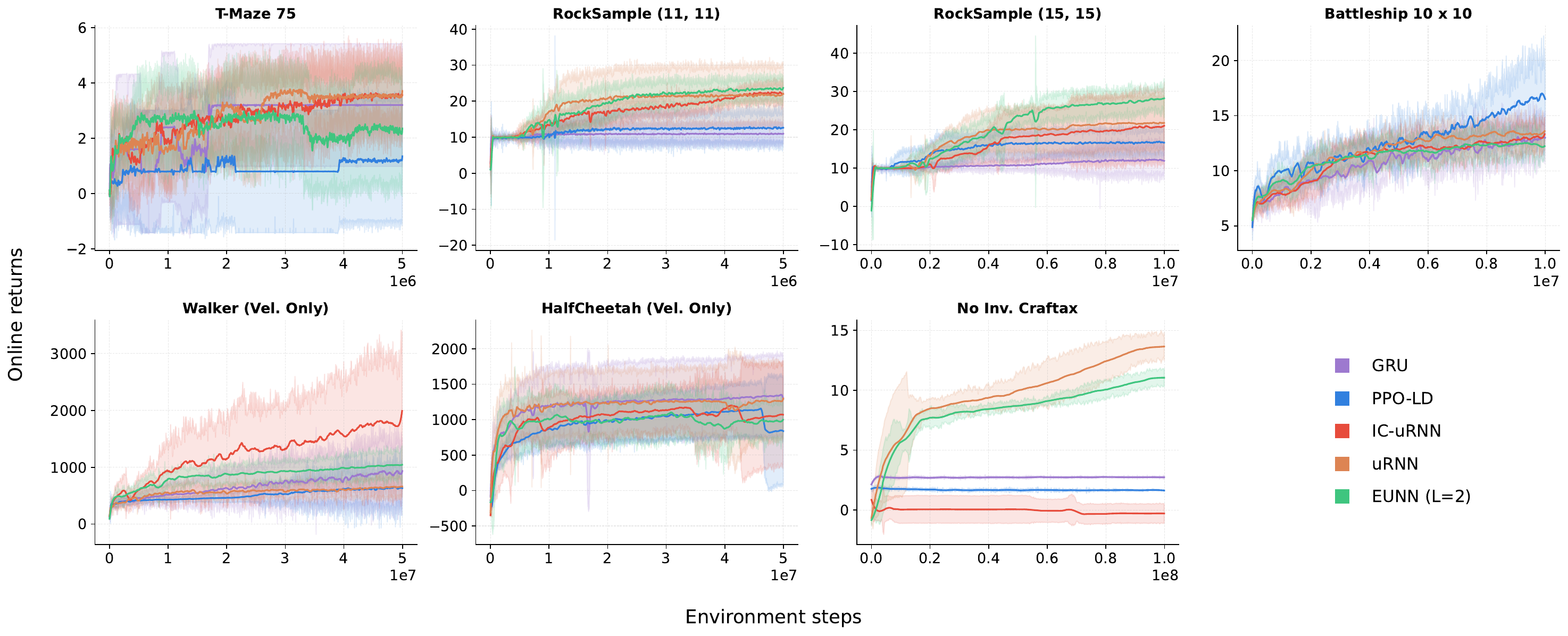}
    \caption{Reward curves for the 3 uRNN variants and the baselines across 7 environments from the POBAX \citep{tao2025pobax} benchmark}
    \label{fig:reward_plot_base}
\end{figure}

\subsection{Results}
The results of the evaluations across the 7 tasks on the 3 unitary recurrent variants are visualized in Fig. \ref{fig:reward_plot_base}. Across the suite, the three unitary variants collectively outperform both the GRU and PPO-LD baselines on a majority of tasks, frequently by a wide margin, and remain competitive on the rest.

The clearest gains appear on the memory-heavy discrete tasks. On both RockSample configurations all three uRNN variants pull well above the baselines, which plateau near a return of $12$: on the $(11,11)$ grid the original uRNN reaches the highest returns, roughly $2\times$ the GRU and PPO-LD plateaus, with EUNN and IC-uRNN following, while on the larger $(15,15)$ grid EUNN fares the best. The largest gains are on No-Inventory Craftax run with downsampled observations, where uRNN and EUNN far exceed the two GRU baselines\footnote{We follow the same observation cropping and downsampling as described in the POBAX \citep{tao2025pobax} paper for all our evaluations but were unable to replicate their results on the Craftax environment}. IC-uRNN is the exception here, failing to make progress on this task despite being the strongest variant elsewhere. We attribute this case to insufficient tuning, given compute budget constraints.

The continuous-control and remaining tasks are more even. On the velocity-only Walker, IC-uRNN stands out with a substantial lead over all other methods, suggesting that conditioning the unitary transition on the observation is particularly useful for integrating velocity history, whereas on HalfCheetah all methods perform comparably. On T-Maze-75, the memory-capable all the methods converge to similar returns, comfortably above the PPO-LD baseline, with EUNN learning fastest early in training. Overall, the complex-recurrent cells deliver consistent and often large improvements on memory-intensive tasks, with different parameterizations excelling on different environments.

% ============================================================================
% 6. PHASE AWARE POLICIES AND QUANTUM INFORMATION 
% ============================================================================

\section{Phase Aware Policies: A Quantum Information Perspective}\label{sec:qm_policy_update}

In the preceding sections we demonstrated that complex-valued representations provide a richer source for learning policies in POMDP settings.
However, we believe that na\"ively treating uRNN cells as a simple drop-in replacement for canonical recurrent units leaves untapped potential in the expressivity of the complex-valued representations.
That is, our approach so far has treated both the real and imaginary parts of the hidden state as real-valued inputs to the policy head. However, a complex vector encodes more than just that - the different dimensions of the vector have different phases and these get lost in the conversion detailed above.
This leads us to ask the question: \textit{Can the phases of the hidden state encode relevant information towards the actions an optimal policy should take at each step?} In the quest of instilling "phase awareness" into our policies, we draw a parallel to concepts from Quantum Information Theory.

Quantum states are primarily represented vectors in a complex eigen basis, that is, they can be expressed as a linear combination of a set of orthogonal eigenvectors. A 'measurement' in this basis, \textit{collapses} the state onto one of the eigenvectors, with a probability determined by the Born rule \citep{born1984quantenmechanik}. A introduction to these concepts and notation used below is included in Appendix \ref{app:qm}. 

%We hypothesize that this numerical mismatch looses valuable information that is encoded in the phase.
%The improved performance seen so far, especially in long horizon memory tasks, further provides first evidence that the inclusion of the phase might serve as a specific memory mechanism.
%This motivates our choice to consider modifications to the policy head such that it is phase-aware and preserves the complex representation through until the final action sampling mechanism.
%This design decision lets us draw an interesting parallel to Quantum Information Theory, in which complex-valued vector spaces are commonly used to represent a quantum-systems state and evolution.
%Particularly relevant to us is that phases in such spaces can interfere constructively and destructively as waves.
%\andr{I don't like this sentence but I don't want to spent too much time refining it now. We can maybe even drop this particular part.}
%Such a property might be highly relevant for memory intensive RL tasks and could potentially allow us to apply RL in real-world scenarios where the markov-property might not always hold.
%\andr{Rough draft end.}

This probabilistic nature of measurement, so fundamental to physical quantities and phenomena, sparks a comparison to the policy we are learning here. Consider a stochastic policy $\pi(a_t \mid s_t)$ over a \textit{discrete} action space $\mathcal{A}$. This policy at any time step $t$ can be expressed as a quantum state vector:
\begin{equation}
    \ket{\pi_t} = \sum_{a \in \mathcal{A}} c^a_t \ket{a}, \quad c^a_t \in \mathbb{C}
\end{equation}

Drawing on the Born rule (Appendix ~\ref{app:qm}), the probability of observing action $a$ is:
\begin{equation}
    \pi(a_t \mid s_t) =  |\braket{a | \pi_t}|^2 = |c^a_t|^2
\end{equation}

Given the complex hidden state $h_t \in \mathbb{C}^n$ from the uRNN, we propose learning the action coefficients directly:
\begin{equation}
    \ket{\pi(s_t)} = \sum_{a \in \mathcal{A}} \hat{c}^a_\theta(h_t) \ket{a}, \quad \hat{c}^a_\theta(h_t) \in \mathbb{C}, \; h_t \in \mathbb{C}^n
\end{equation}

To this end, we replace the real-valued policy head with a complex-valued one with the ModReLU activation. The action coefficients $\hat{c}^a_\theta(h_t)$ are unnormalized. To apply the Born rule while maintaining compatibility with softmax, we output $\log( | \hat{c}^a_\theta(h_t) |^2 )$ as logits:
\begin{equation}
    \pi(a_t \mid s_t) =\frac{\exp (log(| \hat{c}^a_\theta(h_t) |^2))}{\sum_{a \in \mathcal{A}} \exp(log(| \hat{c}^a_\theta(h_t) |^2))} = \frac{| \hat{c}^a_\theta(h_t) |^2}{\sum_{a \in \mathcal{A}} | \hat{c}^a_\theta(h_t) |^2}
\end{equation}

\paragraph{Interference in fully complex policies}
Because the action coefficients are a phase-preserving transformations (ModReLU preserves phase) of the hidden state, for action $a$:
\begin{equation}
    \log \left( | \hat{c}^a_\theta(h_t) |^2 \right) = \log \left( | w^a[0] h_t[0] + w^a[1] h_t[1] + \cdots + w^a[n-1] h_t[n-1] |^2\right)
\end{equation}
where $w^a \in \mathbb{C}^n$ is the weight vector for action $a$. Since both weights and hidden states are complex, the phases interfere constructively or destructively (\eqnref{eq:complex_norm_sum}). Because the hidden state contains information from past time steps and the current observation, this interplay allows for potentially a much richer representation space for the policy.

\paragraph{Unitarity} Furthermore, the evolution of quantum states through time, are also constrained by Unitary transformations. While our recurrent dynamics are not exactly unitary, owing to the addition of the observation dependent vector, this connection between concepts of reinforcement learning and quantum information theory is noteworthy.

\subsection{Experiments: Phase Aware Policy uRNNs}\label{sec:experiment_2}

Following the methodology described above, we integrate the fully complex policy head with the Born rule based sampling into the 3 previously specified uRNN variants, to have: QuRNN, QICuRNN and QEUNN. As we described the method for discrete action spaces, we limit our evaluation of the QuRNN variants to exclude the Walker and HalfCheetah environments. All other settings and hyperparameters (Appendix \ref{app:hyperparameters}) stay the same as Section \ref{sec:exp_setup}.

\begin{figure}[t]
    \centering
    \includegraphics[width=0.85\textwidth]{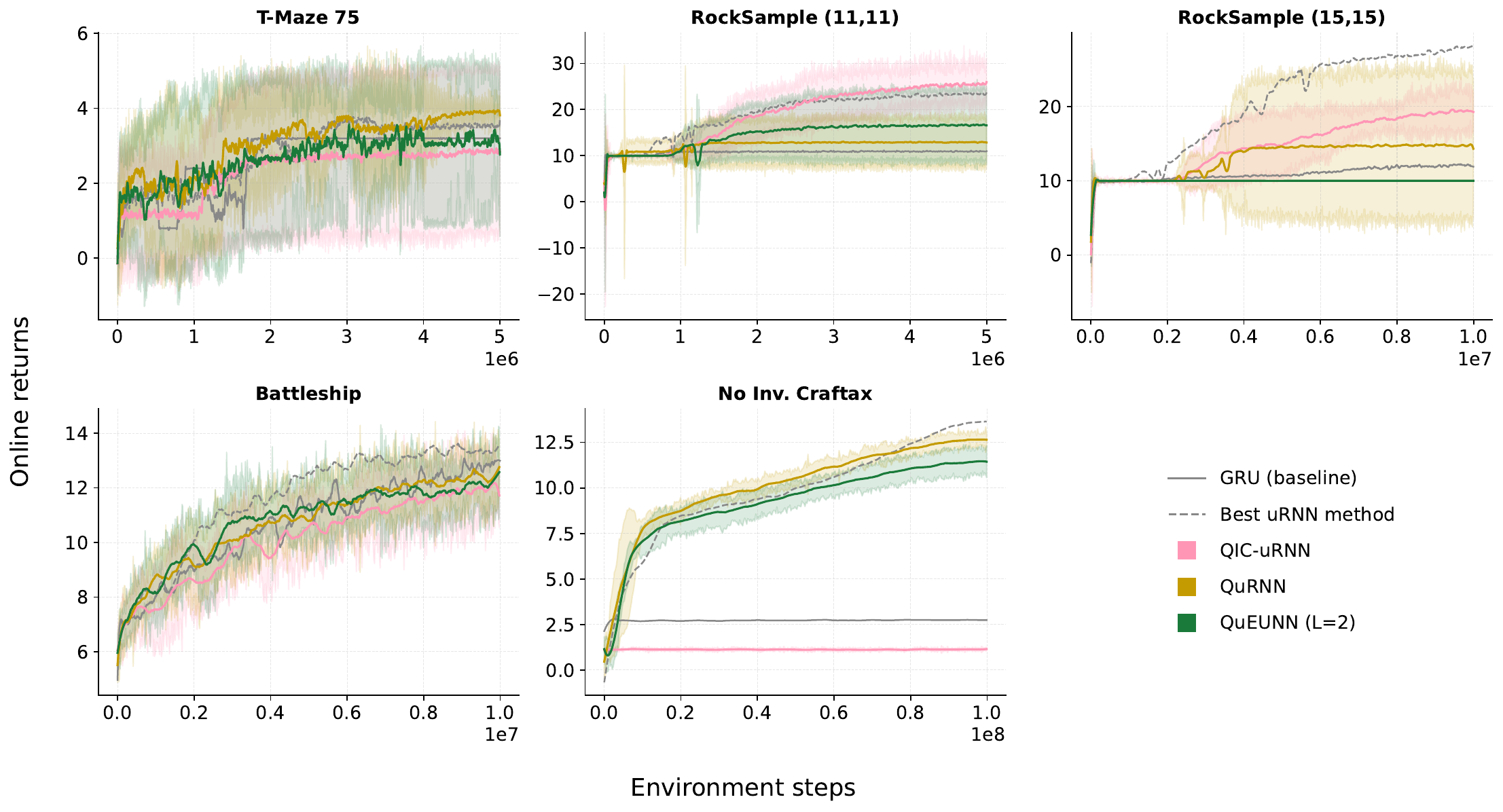}
    \caption{Reward curves for the 3 Phase Aware QuRNN variants plotted with a GRU baseline and the best uRNN method on 5 environments with discrete action spaces}
    \label{fig:reward_plot_born}
\end{figure}

\paragraph{Results.} The reward curves for the three phase-aware variants are shown in \figref{fig:reward_plot_born}, alongside the GRU baseline and the best base uRNN method from Section~\ref{sec:experiments} for clarity. Overall, the phase-aware policies are competitive with their base uRNN counterparts on most tasks. The performance is  strongest again on No-Inventory Craftax, where QuRNN and QuEUNN reach close to $5\times$ the GRU return and track the best base uRNN, and on RockSample, where QIC-uRNN slightly exceeds the best base uRNN variant. On T-Maze-75 and Battleship evironments, all methods end up tightly clustered around the GRU baseline. The exception mirrors the base results: QIC-uRNN fails to learn on Craftax, just as IC-uRNN did, while remaining one of the stronger variants on RockSample. These results indicate that preserving the phase through to action sampling retains the benefits of the complex-valued recurrence without sacrificing performance relative to the real-valued policy heads.

% ============================================================================
% 7. DISCUSSION AND CONCLUSION
% ============================================================================
\section{Discussion and Conclusion}\label{sec:disc_conc}

% I guess we dont need this section actually
% \textbf{Summary of Contributions and Results.} We revisited the uRNN architecture, showed how to integrate it into PPO, introduced an input-conditioned formulation (ICuRNN), and proposed a quantum-inspired complex-valued policy head with Born rule sampling. Our evaluations yield the following findings:

% \begin{itemize}
%     \item[Q1.] The uRNN architecture succeeds at memory tasks, with IC-uRNN-PPO being the most sample efficient on S7. Both uRNN variants solve S13 on some seeds while the LSTM baseline does not. On Atari Breakout, uRNN variants narrowly beat the LSTM.
%     \item[Q2.] The Q-uRNN-PPO and QIC-uRNN-PPO variants with fully complex policy heads are the only methods that completely solve Memory S13, demonstrating the benefits of complex-valued policies with phase interference.
%     \item[Q3.] Hidden state norm analysis shows consistent patterns of information accumulation, supporting the argument for unitary recurrence.
% \end{itemize}

% here would be the additions I guess

In this work, we revisited unitary evolution recurrent neural networks as a representation for reinforcement learning in partially observable environments. We integrated their complex-valued, norm-preserving recurrence into recurrent PPO as a drop-in replacement for the GRU cell, and studied three parameterizations of the unitary transition. We then proposed a phase-aware policy head that samples actions through a Born-rule mechanism, preserving the complex representation through to action selection.
The strong performance of our methods across the POBAX benchmark are evidence that complex-valued representations and unitary dynamics are a promising new direction for representation learning in RL, rather than a niche architectural choice. The connection to quantum information theory further hints at a framework bridging the two fields, as previously explored in a different context by \citet{dong2008quantumreinforcementlearning}. We believe these findings motivate a deeper investigation of complex-valued states for memory and partial observability in RL

\textbf{Limitations.} We do not explore the hyperparameter space in depth. Given the known sensitivity of PPO to hyperparameters \citep{pmlr-v202-eimer23a, shengyi2022the37implementation}, more thorough exploration is needed. Our experiments are limited to part of the POBAX benchmark; evaluation on benchmarks like Memory Gym \citep{pleines2023memory} and Memory Maze \citep{pasukonis2022memmaze} would strengthen the argument. While we demonstrate the performance of the base version of the EUNN architecture \citep{jing2017tunablenetworks}, we only evaluate the simplest tuning setting as a preliminary set of experiments, and we hope to understand the impact of the tuning parameter, and the performance of the method when accessing larger unitary spaces in the future.

\textbf{Future Work.} Looking to the future, we hope to extend this work in two main directions. The dual representation of the complex hidden state naturally motivates an exciting question: \textit{what do the phases of the hidden states encode throughout the rollout?} We hope to investigate this mechanistically in upcoming work.
A natural extension is to apply complex-valued representations and unitary transformations to model-based RL, where the norm-preserving property becomes central to state transitions in learned world models. More diverse environments and deeper analysis of training dynamics are also important directions.

% ============================================================================
% ACKNOWLEDGMENTS (hidden in anonymous submission)
% ============================================================================
\begin{ack}
André Biedenkapp is funded by the Deutsche Forschungsgemeinschaft (DFG, German
Research Foundation) – 572775489.
The authors acknowledge support by the state of Baden-Württemberg through bwHPC and the German Research Foundation (DFG) through grant INST 35/1597-1 FUGG. Efstratios Gavves and Sathya Kamesh Bhethanabhotla are supported by the European Union’s Horizon Europe research and innovation programme under grant agreement number 101214398 (ELLIOT).
\end{ack}

% ============================================================================
% REFERENCES
% ============================================================================
\bibliographystyle{unsrtnat}
\bibliography{lib}

% ============================================================================
% APPENDIX
% ============================================================================
\appendix

\section{Complex Forward and Backward Passes}\label{sec:cvnns}

When uRNNs were first proposed, PyTorch \citep{paszke2019pytorch} and TensorFlow \citep{tensorflow2015-whitepaper} did not offer reliable complex data types, so researchers had to write custom layers for initialization, normalization, and backpropagation. Works like \citep{trabelsi2018deepcomplex} still existed to describe the mechanics of complex-valued neural networks, defined activations, layers and normalization techniques around them, while demonstrating their applicability in various domains. \par

Today the situation is much simpler. Both PyTorch and JAX \citep{jax2018github} include built-in complex dtypes (for example \texttt{torch.complex64} and \texttt{jax.numpy.complex64}) and their autograd systems support Wirtinger calculus automatically \citep{amin2011wirtinger}. This means a complex forward or backward pass can be written just like its real-valued counterpart: declare parameters in a complex dtype and the framework handles gradient propagation. \par
In this work, we rely heavily on these frameworks for a much simpler and straightforward implementation of complex-valued neural networks.

\section{Complex Algebra and Unitary Matrices}\label{sec:compalg}

An important idea in this work is the extension of linear algebra to the complex space. Complex numbers and complex-valued representations provide an additional degree of freedom in the form of a phase.

A complex number $z = x + iy$ can also be represented as $z = r e^{i\theta}$ with $r = |z|$ and $\theta = \tan^{-1}(y/x)$. The product of two complex numbers can be seen as rotations in phase space: $z_1 z_2 = r_1 r_2 e^{i(\theta_1 + \theta_2)}$.

An important interaction is the squared norm of the sum of two complex numbers:
\begin{equation}
| z_1 + z_2 |^2 = |z_1|^2 + |z_2|^2 + 2 |z_1||z_2| \cos(\theta_2 - \theta_1)
\label{eq:complex_norm_sum}
\end{equation}
This shows that the squared norm depends on the phase difference, a property we will exploit later.

A complex vector space $\mathbb{C}^n$ is a Hilbert space where the inner product takes the form $\langle a, b \rangle = a^\ast b$, with $(\cdot)^\ast$ denoting conjugate transpose.

Now for matrices in a Hilbert space, the unitary property is defined as follows: 
\begin{definition}
A matrix $U \in \mathbb{C}^{n \times n}$ is \emph{unitary} if $U^\dagger U = I$, where $U^\dagger$ denotes the conjugate transpose. For any vector $x \in \mathbb{C}^n$: $\|Ux\|_2 = \|x\|_2$.
\end{definition}
All eigenvalues of a unitary matrix lie on the complex unit circle ($|\lambda| = 1$).

\paragraph{Householder Reflections.}\label{par:householder}
A Householder reflection (Fig. \ref{fig:Householder}), characterized by a matrix in complex space is $H = I - 2\,vv^\ast/(v^\ast v)$, $v \in \mathbb{C}^n$. All Householder matrices are unitary (see Appendix~\ref{app:householder_unitarity}), and composing multiple reflections generates a rich family of unitary transformations efficiently. This idea forms the core of how we can parameterize unitary recurrent systems.

\section{Proof for the unitarity of Householder Matrices}\label{app:householder_unitarity}

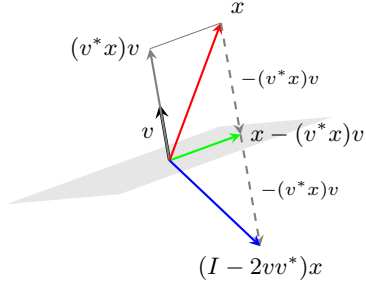
\begin{figure}[t]
    \begin{center}
    \begin{tikzpicture}[x={(20:10mm)}, y={(100:15mm)}, z={(5:10mm)}, >=stealth]
        % Draw the hyperplane (plane orthogonal to v) - made smaller
        \fill[black!10] (-1.5,0,-0.75) -- (-1.5,0,0.75) --  (1.5,0,0.75) -- (1.5,0,-0.75) -- cycle;
        
        % Define coordinates
        \coordinate (o) at (0,0,0);
        \coordinate (a) at (0,1,0);
        \coordinate (b) at (1,0,0);
        \coordinate (c) at (1,-1,0);
        \coordinate (x) at (1,1,0);
        \coordinate (v) at (0,0.5,0);  % Made v bigger
        
        % Draw vector v (normal to the hyperplane) - made bigger
        \draw[very thick, ->] (o) -- (v) node[midway, left] {$v$};
        
        % Draw projection of x onto v: (v^H x)v
        \draw[thick, ->, gray] (o) -- (a) node[black, left] {$(v^* x)v$};
        
        % Draw component perpendicular to v: x - (v^H x)v
        \draw[thick, ->, green] (o) -- (b) node[black, right] {$x-(v^* x)v$};
        
        % Draw original vector x
        \draw[thick, ->, red] (o) -- (x) node[black, above right] {$x$};
        
        % Draw dashed line from x to the perpendicular component
        \draw[gray] (a) -- (x);
        
        % Draw first dashed arrow showing -(v^H x)v
        \draw[thick, ->, gray, dashed] (x) -- (b) node[black, midway, right, font=\scriptsize] {$-(v^* x)v$};
        
        % Draw second dashed arrow showing -(v^H x)v again
        \draw[thick, ->, gray, dashed] (b) -- (c) node[black, midway, right, font=\scriptsize] {$-(v^* x)v$};
        
        % Draw reflected vector Hx = (I - 2vv^H)x
        \draw[thick, ->, blue] (o) -- (c) node[black, below] {$(I-2vv^*)x$};
        
    \end{tikzpicture}
    \end{center}
    \caption[Geometric interpretation of a Householder reflection]{\textbf{Geometric interpretation of a Householder reflection.} A vector $x$ is reflected across the hyperplane orthogonal to $v$ to produce $Hx$, where $H = I - 2vv^\ast/(v^\ast v)$ is the Householder matrix. The reflection flips the parallel component while preserving the perpendicular component.}
    \label{fig:Householder}
\end{figure}

In \ref{par:householder}, we introduced the Householder reflection as a way to parametrize unitary matrices. We now provide the proof for the unitarity of Householder matrices.

\begin{theorem}
Let $v \in \mathbb{C}^n$ be a nonzero vector. The Householder matrix defined by
\[
H = I - \frac{2vv^\dagger}{v^\dagger v}
\]
is unitary, i.e., $H^\dagger H = I$.
\end{theorem}

\begin{proof}
We verify that $H^\dagger H = I$ by direct computation.

First, we compute the Hermitian conjugate of $H$:
\begin{align*}
H^\dagger &= \left(I - \frac{2vv^*}{v^* v}\right)^\dagger \\
&= I^\dagger - \frac{2}{v^* v}\left(vv^*\right)^\dagger \\
&= I - \frac{2}{(v^* v)^*}\left(v^*\right)^* v^* \\
&= I - \frac{2}{v^* v}vv^* \\
&= H,
\end{align*}
where we used the fact that $v^* v$ is real and positive (hence $(v^* v)^* = v^* v$), and $\left(v^*\right)^* = v$.

Thus, $H$ is Hermitian. Now we verify unitarity:
\begin{align*}
H^\dagger H &= H^2 \\
&= \left(I - \frac{2vv^*}{v^* v}\right)\left(I - \frac{2vv^*}{v^* v}\right) \\
&= I - \frac{2vv^*}{v^* v} - \frac{2vv^*}{v^* v} + \frac{4vv^* vv^*}{(v^* v)^2} \\
&= I - \frac{4vv^*}{v^* v} + \frac{4v(v^* v)v^*}{(v^* v)^2} \\
&= I - \frac{4vv^*}{v^* v} + \frac{4vv^*}{v^* v} \\
&= I,
\end{align*}
where in the second-to-last line we used the fact that $v^* v$ is a scalar, so $v(v^* v) = (v^* v)v$.

Therefore, $H^\dagger H = I$, which proves that $H$ is unitary.
\end{proof}

\section{Quantum Information Theory and Measurements}\label{app:qm}

This section serves as a brief introduction to the concepts of quantum information theory and measurements, which we will use to motivate a alternative perspective from which this work and the underlying Unitary Recurrence can be viewed, in Section~\ref{sec:qm_policy_update}. While this section only covers the concepts and relationships relevant to this work, a more comprehensive treatment of the complex linear algebra and quantum mechanical principles can be found in \citep{Nielsen_Chuang_2010}. \par
Throughout this section, we use Dirac notation: a ket $\ket{\psi}$ denotes a column vector, while a bra $\bra{\phi}$ denotes a row vector. A braket $\braket{\phi, \psi}$ denotes the inner product $\bra \phi\ket{\psi}$, and a ketbra $\ket{\psi}\bra{\phi}$ denotes the outer product.\par \hspace{1em}

A quantum state provides a complete description of an isolated physical system.
\begin{definition}
    In a finite-dimensional Hilbert space $\mathcal{H} \cong \mathbb{C}^d$, any pure state is represented as a unit vector
    \[
    \ket{\psi} = \sum_{k=1}^d c_k \ket{k},
    \]
    where $\{\ket{k}\}$ is an orthonormal basis and the coefficients $c_k \in \mathbb{C}$ satisfy $\sum_k |c_k|^2 = 1$.
\end{definition}
The definition of basis is the same as in real vector spaces, where any set of orthogonal vectors that span the space can be a basis.
This state vector expansion is not merely a coordinate representation: it encodes all physically accessible information about the system. \par

A state vector written in this form as a linear combination of the basis eigenstates, is said to be in a \emph{superposition} of the eigenstates. The coefficients $c_k$ are the \emph{amplitudes} of the state vector in the basis eigenstates, and if all the amplitudes $c_k$ are equal, then this superposition is a \emph{equal superposition state}, meaning the state vector is equally likely to be in any of the basis eigenstates.

\subsection{Operators and State Transformations}

In quantum mechanics, transformations of quantum states are represented by \emph{Operators}, which are square matrices acting on the Hilbert space. An operator $A \in \mathbb{C}^{d \times d}$ transforms a state $\ket{\psi}$ to $A\ket{\psi}$, analogous to how a matrix multiplies a vector in linear algebra. Operators can represent various physical processes: rotations, measurements, or time evolution.

For a measurement to be physically meaningful, the corresponding operator must be \emph{Hermitian} (as defined in \ref{sec:compalg}). Hermitian operators have real eigenvalues and form a complete orthonormal eigenbasis, making them suitable for representing physical observables, or quantities that can be measured, such as energy or position. When we measure an observable represented by a Hermitian operator $M$, we are effectively projecting the state onto one of $M$'s eigenstates.

\subsection{Measurement and Probabilistic Outcomes}

A measurement in the eigenbasis $\{\ket{k}\}$ of a Hermitian operator is inherently probabilistic. Upon measurement, the outcome $\lambda_k$ (corresponding to eigenstate $\ket{k}$) is observed with probability
\[
p(k) = |\braket{k, \psi}|^2 = |c_k|^2,
\]
as prescribed by the Born rule, and the post-measurement state collapses to $\ket{k}$. The stochasticity arises not from noise but from the fundamental quantum mechanical property of projecting the state onto an eigenbasis of the measured observable. This projection is a linear operation that transforms the state from a superposition $\ket{\psi} = \sum_k c_k \ket{k}$ to a single basis state $\ket{k}$.

\subsection{Time Evolution and Unitary Dynamics}

The time evolution of a closed quantum system is governed by the Schrödinger equation, which is a linear differential equation that describes how the state of a quantum system changes over time.
\[
i \hbar \frac{d}{dt} \ket{\psi} = H \ket{\psi},
\]
where $H$ is the Hamiltonian, a Hermitian operator encoding the system's energy structure.  
Its formal solution is
\[
\ket{\psi(t)} = U(t)\ket{\psi(0)}, \qquad U(t) = e^{-\frac{i}{\hbar}Ht}.
\]
where $U(t)$ is the time evolution operator. Because $H$ is Hermitian, the operator $U(t)$ is unitary, ensuring that the norm of the state is preserved during evolution.  
This norm preservation mirrors the dynamical behaviour of unitary recurrent updates used in machine learning: the global “energy” of the representation remains constant, and information is carried forward through rotations in complex space rather than amplification or decay.

\section{Hyperparameters}\label{app:hyperparameters}

We organize the hyperparameters into three tables. Table~\ref{tab:shared_hparams} lists the PPO settings that are held fixed across the two baselines (GRU, PPO-LD) and all six uRNN variants. Table~\ref{tab:env_settings} lists the per-environment settings that are fixed across all methods but vary by environment (hidden size, parallel environments, training budget, and discount). Table~\ref{tab:baseline_hparams} reports the per-environment best hyperparameters of the two baselines, found by sweeping over a grid.

% ---------------------------------------------------------------------------
\begin{table}[h]
\centering
\caption{Hyperparameters held fixed across all baselines and all six uRNN variants. The rollout length $T$ is the only PPO setting we change relative to the baselines (see note).}
\label{tab:shared_hparams}
\begin{tabular}{lcl}
\toprule
\textbf{Hyperparameter} & \textbf{Value} & \textbf{Note} \\
\midrule
Optimizer & Adam & \\
Update epochs & 4 & per PPO update \\
Minibatches per update & 4 & \\
PPO clip $\epsilon$ & 0.2 & \\
Value loss coefficient $c_v$ & 0.5 & \\
Max gradient norm & 0.5 & global-norm clipping \\
Learning-rate anneal & linear $\to 0$ & applied to the real parameter group \\
Action concatenation & yes & previous action $a_{t-1}$ appended to $o_t$ \\
Rollout length $T$ & 128 / 256 & 128 for baselines, 256 for uRNN variants; \\
                    &           & Craftax uses 64 \\
Seeds & 5 \\
\bottomrule
\end{tabular}
\end{table}

% ---------------------------------------------------------------------------
\begin{table}[h]
\centering
\caption{Per-environment settings, fixed across all methods. The complex hidden state of each uRNN variant has the same dimension $n$ as the GRU baseline's hidden state. Training budgets follow the POBAX benchmark \citep{tao2025pobax}; Craftax is capped at $10^8$ steps.}
\label{tab:env_settings}
\begin{tabular}{lcccc}
\toprule
\textbf{Environment} & \textbf{Hidden size $n$} & \textbf{Parallel envs} & \textbf{Total steps} & \textbf{$\gamma$} \\
\midrule
T-Maze (corridor 75)      & 32  & 4   & $5.0\times10^{6}$ & 0.99 \\
RockSample $(11,11)$      & 256 & 8   & $5.0\times10^{6}$ & 0.99 \\
RockSample $(15,15)$      & 512 & 16  & $1.0\times10^{7}$ & 0.999 \\
Battleship $(10\times10)$ & 512 & 32  & $1.0\times10^{7}$ & 1.0 \\
Masked Walker             & 256 & 4   & $5.0\times10^{7}$ & 0.99 \\
Masked HalfCheetah        & 256 & 4   & $5.0\times10^{7}$ & 0.99 \\
Craftax (no inventory)    & 512 & 256 & $1.0\times10^{8}$ & 0.99 \\
\bottomrule
\end{tabular}
\end{table}

% ---------------------------------------------------------------------------
\begin{table}[h]
\centering
\caption{Per-environment best hyperparameters for the baselines, selected by a grid sweep, exactly following \citep{tao2025pobax}}
\label{tab:baseline_hparams}
\begin{tabular}{lccc cccccc}
\toprule
 & \multicolumn{3}{c}{\textbf{GRU-PPO}} & \multicolumn{5}{c}{\textbf{PPO-LD}} \\
\cmidrule(lr){2-4}\cmidrule(lr){5-9}
\textbf{Environment} & lr & ent. & $\lambda_0$ & lr & ent. & $\lambda_0$ & $\lambda_1$ & $\beta$ \\
\midrule
T-Maze (corridor 75)      & $2.5\!\times\!10^{-3}$ & 0.01 & 0.7  & $2.5\!\times\!10^{-4}$ & 0.01 & 0.95 & 0.95 & 0.25 \\
RockSample $(11,11)$      & $2.5\!\times\!10^{-3}$ & 0.2  & 0.7  & $2.5\!\times\!10^{-3}$ & 0.2  & 0.5  & 0.5  & 0.25 \\
RockSample $(15,15)$      & $2.5\!\times\!10^{-3}$ & 0.2  & 0.7  & $2.5\!\times\!10^{-3}$ & 0.2  & 0.1  & 0.95 & 0.5 \\
Battleship $(10\times10)$ & $2.5\!\times\!10^{-3}$ & 0.05 & 0.7  & $2.5\!\times\!10^{-3}$ & 0.05 & 0.1  & 0.95 & 0.5 \\
Masked Walker             & $2.5\!\times\!10^{-4}$ & 0.01 & 0.95 & $2.5\!\times\!10^{-4}$ & 0.01 & 0.95 & 0.95 & 0.5 \\
Masked HalfCheetah        & $2.5\!\times\!10^{-4}$ & 0.01 & 0.9  & $2.5\!\times\!10^{-5}$ & 0.01 & 0.95 & 0.7  & 0.25 \\
Craftax (no inventory)    & $2.5\!\times\!10^{-4}$ & 0.01 & 0.5  & $2.5\!\times\!10^{-4}$ & 0.01 & 0.1  & 0.95 & 0.25 \\
\bottomrule
\end{tabular}
\end{table}

\subsection{Per-environment sweeps for the uRNN variants}\label{app:our_hparams}

We tune four hyperparameters per environment for our six variants: the real and complex learning rates (lr, complex lr), the entropy coefficient, and the GAE coefficient $\lambda_0$. All remaining settings are shared with the baselines (Tables~\ref{tab:shared_hparams} and~\ref{tab:env_settings}); the second GAE coefficient $\lambda_1=0.5$ and the $\lambda$-discrepancy weight are not used ($\beta=0$, single critic). For each environment we report the search space (left) and the best setting per variant (right). For \textbf{Craftax} we do not tune the uRNN variants: we reuse the baseline settings of Table~\ref{tab:baseline_hparams}.

% ---------------------------------------------------------------------------
\begin{table}[h]
\centering
\caption{\textbf{T-Maze (corridor 75).} Search space (left) and best per-variant settings (right).}
\label{tab:sweep_tmaze}
\begin{minipage}[t]{0.42\textwidth}\centering\small
\begin{tabular}{l l}
\toprule
\textbf{Parameter} & \textbf{Values swept} \\
\midrule
complex lr & $10^{-6},5\!\times\!10^{-6},10^{-5},$ \\
           & $5\!\times\!10^{-5},10^{-4}$ \\
entropy    & $0.005,\,0.01,\,0.05$ \\
$\lambda_0$ & $0.8,\,0.9,\,0.95$ \\
lr         & $2.5\!\times\!10^{-4},\,2.5\!\times\!10^{-3}$ \\
\bottomrule
\end{tabular}
\end{minipage}\hfill
\begin{minipage}[t]{0.56\textwidth}\centering\small
\begin{tabular}{l cccc}
\toprule
\textbf{Variant} & complex lr & ent. & $\lambda_0$ & lr \\
\midrule
uRNN     & $10^{-6}$ & 0.005 & 0.95 & $2.5\!\times\!10^{-4}$ \\
EUNN     & $5\!\times\!10^{-5}$ & 0.01 & 0.8 & $2.5\!\times\!10^{-4}$ \\
ICuRNN   & $5\!\times\!10^{-5}$ & 0.01 & 0.95 & $2.5\!\times\!10^{-4}$ \\
QuRNN    & $10^{-5}$ & 0.005 & 0.95 & $2.5\!\times\!10^{-3}$ \\
QEUNN    & $10^{-4}$ & 0.005 & 0.8 & $2.5\!\times\!10^{-4}$ \\
QICuRNN  & $5\!\times\!10^{-6}$ & 0.01 & 0.9 & $2.5\!\times\!10^{-3}$ \\
\bottomrule
\end{tabular}
\end{minipage}
\end{table}

% ---------------------------------------------------------------------------
\begin{table}[h]
\centering
\caption{\textbf{RockSample $(11,11)$.} Search space (left) and best per-variant settings (right).}
\label{tab:sweep_rs11}
\begin{minipage}[t]{0.42\textwidth}\centering\small
\begin{tabular}{l l}
\toprule
\textbf{Parameter} & \textbf{Values swept} \\
\midrule
complex lr & $10^{-6},\,10^{-5},\,8\!\times\!10^{-5}$ \\
entropy    & $0.075,\,0.1,\,0.15,\,0.2$ \\
$\lambda_0$ & $0.7,\,0.8,\,0.95$ \\
lr         & $2.5\!\times\!10^{-4},\,2.5\!\times\!10^{-3}$ \\
\bottomrule
\end{tabular}
\end{minipage}\hfill
\begin{minipage}[t]{0.56\textwidth}\centering\small
\begin{tabular}{l cccc}
\toprule
\textbf{Variant} & complex lr & ent. & $\lambda_0$ & lr \\
\midrule
uRNN     & $8\!\times\!10^{-5}$ & 0.075 & 0.7 & $2.5\!\times\!10^{-3}$ \\
EUNN     & $10^{-6}$ & 0.15 & 0.7 & $2.5\!\times\!10^{-3}$ \\
ICuRNN   & $10^{-6}$ & 0.075 & 0.7 & $2.5\!\times\!10^{-3}$ \\
QuRNN    & $8\!\times\!10^{-5}$ & 0.075 & 0.7 & $2.5\!\times\!10^{-3}$ \\
QEUNN    & $8\!\times\!10^{-5}$ & 0.1 & 0.7 & $2.5\!\times\!10^{-3}$ \\
QICuRNN  & $8\!\times\!10^{-5}$ & 0.15 & 0.7 & $2.5\!\times\!10^{-4}$ \\
\bottomrule
\end{tabular}
\end{minipage}
\end{table}

% ---------------------------------------------------------------------------
\begin{table}[h]
\centering
\caption{\textbf{RockSample $(15,15)$.} Search space (left) and best per-variant settings (right).}
\label{tab:sweep_rs15}
\begin{minipage}[t]{0.42\textwidth}\centering\small
\begin{tabular}{l l}
\toprule
\textbf{Parameter} & \textbf{Values swept} \\
\midrule
complex lr & $10^{-6},\,10^{-5},\,8\!\times\!10^{-5}$ \\
entropy    & $0.075,\,0.1,\,0.15,\,0.2$ \\
$\lambda_0$ & $0.7,\,0.8,\,0.95$ \\
lr         & $2.5\!\times\!10^{-4},\,2.5\!\times\!10^{-3}$ \\
\bottomrule
\end{tabular}
\end{minipage}\hfill
\begin{minipage}[t]{0.56\textwidth}\centering\small
\begin{tabular}{l cccc}
\toprule
\textbf{Variant} & complex lr & ent. & $\lambda_0$ & lr \\
\midrule
uRNN     & $10^{-5}$ & 0.075 & 0.7 & $2.5\!\times\!10^{-3}$ \\
EUNN     & $10^{-6}$ & 0.1 & 0.7 & $2.5\!\times\!10^{-3}$ \\
ICuRNN   & $8\!\times\!10^{-5}$ & 0.15 & 0.7 & $2.5\!\times\!10^{-4}$ \\
QuRNN    & $8\!\times\!10^{-5}$ & 0.1 & 0.7 & $2.5\!\times\!10^{-3}$ \\
QEUNN    & $10^{-6}$ & 0.075 & 0.8 & $2.5\!\times\!10^{-3}$ \\
QICuRNN  & $8\!\times\!10^{-5}$ & 0.2 & 0.7 & $2.5\!\times\!10^{-4}$ \\
\bottomrule
\end{tabular}
\end{minipage}
\end{table}
% ---------------------------------------------------------------------------
\begin{table}[h]
\centering
\caption{\textbf{Battleship $(10\times10)$.} Search space (left) and best per-variant settings (right).}
\label{tab:sweep_battleship}
\begin{minipage}[t]{0.42\textwidth}\centering\small
\begin{tabular}{l l}
\toprule
\textbf{Parameter} & \textbf{Values swept} \\
\midrule
complex lr & $10^{-6},\,10^{-5},\,8\!\times\!10^{-5}$ \\
entropy    & $0.01,\,0.05,\,0.1$ \\
$\lambda_0$ & $0.6,\,0.7,\,0.8,\,0.9,\,0.95$ \\
lr         & $2.5\!\times\!10^{-3}$ \\
\bottomrule
\end{tabular}
\end{minipage}\hfill
\begin{minipage}[t]{0.56\textwidth}\centering\small
\begin{tabular}{l cccc}
\toprule
\textbf{Variant} & complex lr & ent. & $\lambda_0$ & lr \\
\midrule
uRNN     & $8\!\times\!10^{-5}$ & 0.01 & 0.9 & $2.5\!\times\!10^{-3}$ \\
EUNN     & $10^{-6}$ & 0.01 & 0.7 & $2.5\!\times\!10^{-3}$ \\
ICuRNN   & $8\!\times\!10^{-5}$ & 0.01 & 0.95 & $2.5\!\times\!10^{-3}$ \\
QuRNN    & $10^{-5}$ & 0.01 & 0.8 & $2.5\!\times\!10^{-3}$ \\
QEUNN    & $10^{-5}$ & 0.01 & 0.8 & $2.5\!\times\!10^{-3}$ \\
QICuRNN  & $10^{-5}$ & 0.01 & 0.9 & $2.5\!\times\!10^{-3}$ \\
\bottomrule
\end{tabular}
\end{minipage}
\end{table}

% ---------------------------------------------------------------------------
\begin{table}[h]
\centering
\caption{\textbf{Masked Walker.} Search space (left) and best per-variant settings (right). The Born-rule (Q) variants are defined for discrete action spaces only and are not evaluated here.}
\label{tab:sweep_walker}
\begin{minipage}[t]{0.42\textwidth}\centering\small
\begin{tabular}{l l}
\toprule
\textbf{Parameter} & \textbf{Values swept} \\
\midrule
complex lr & $10^{-6},\,10^{-5},\,8\!\times\!10^{-5}$ \\
entropy    & $0.005,\,0.01,\,0.05,\,0.1$ \\
$\lambda_0$ & $0.8,\,0.9,\,0.95$ \\
lr         & $2.5\!\times\!10^{-4}$ \\
\bottomrule
\end{tabular}
\end{minipage}\hfill
\begin{minipage}[t]{0.56\textwidth}\centering\small
\begin{tabular}{l cccc}
\toprule
\textbf{Variant} & complex lr & ent. & $\lambda_0$ & lr \\
\midrule
uRNN     & $8\!\times\!10^{-5}$ & 0.01 & 0.95 & $2.5\!\times\!10^{-4}$ \\
EUNN     & $10^{-5}$ & 0.01 & 0.95 & $2.5\!\times\!10^{-4}$ \\
ICuRNN   & $8\!\times\!10^{-5}$ & 0.005 & 0.95 & $2.5\!\times\!10^{-4}$ \\
\bottomrule
\end{tabular}
\end{minipage}
\end{table}

% ---------------------------------------------------------------------------
\begin{table}[h]
\centering
\caption{\textbf{Masked HalfCheetah.} Search space (left) and best per-variant settings (right). The Born-rule (Q) variants are defined for discrete action spaces only and are not evaluated here.}
\label{tab:sweep_halfcheetah}
\begin{minipage}[t]{0.42\textwidth}\centering\small
\begin{tabular}{l l}
\toprule
\textbf{Parameter} & \textbf{Values swept} \\
\midrule
complex lr & $10^{-6},\,10^{-5},\,8\!\times\!10^{-5}$ \\
entropy    & $0.005,\,0.01,\,0.05,\,0.1$ \\
$\lambda_0$ & $0.8,\,0.9,\,0.95$ \\
lr         & $2.5\!\times\!10^{-4}$ \\
\bottomrule
\end{tabular}
\end{minipage}\hfill
\begin{minipage}[t]{0.56\textwidth}\centering\small
\begin{tabular}{l cccc}
\toprule
\textbf{Variant} & complex lr & ent. & $\lambda_0$ & lr \\
\midrule
uRNN     & $8\!\times\!10^{-5}$ & 0.005 & 0.9 & $2.5\!\times\!10^{-4}$ \\
EUNN     & $10^{-6}$ & 0.005 & 0.9 & $2.5\!\times\!10^{-4}$ \\
ICuRNN   & $8\!\times\!10^{-5}$ & 0.01 & 0.8 & $2.5\!\times\!10^{-4}$ \\
\bottomrule
\end{tabular}
\end{minipage}
\end{table}

\end{document}